\documentclass[11pt,twocolumn]{IEEEtran}
\usepackage[margin=0.75in]{geometry}
\usepackage{amsmath,amssymb}
\usepackage{graphicx}   
\usepackage{color}      
\usepackage{epstopdf}
\usepackage{cite}
\usepackage{amssymb,amsthm}
\usepackage[font=footnotesize]{subfig}

\newtheorem{prob-statement}{Problem}
\newtheorem{lemma}{Lemma}

\newtheorem{property}{Property}
\newtheorem{thrm}{Theorem}

\begin{document}

\title{Towards the Design of Prospect-Theory based Human Decision Rules for Hypothesis Testing}

\author{V.~Sriram~Siddhardh~Nadendla,~\IEEEmembership{Student Member,~IEEE}, Swastik~Brahma,~\IEEEmembership{Member,~IEEE}, and~Pramod~K.~Varshney,~\IEEEmembership{Fellow,~IEEE}
\thanks{V. Sriram Siddhardh Nadendla, Swastik Brahma and Pramod K. Varshney are with the Department of Electrical Engineering and Computer Science, Syracuse University, Syracuse, NY 13201, USA. E-mail: \{vnadendl, skbrahma, varshney\}@syr.edu.}
\thanks{This work was supported in part by CASE: The Center for Advanced Systems and Engineering, a NYSTAR center for advanced technology at Syracuse University, the Department of the Army under the Cooperative Research Agreement: W911-NF-13-2-0040, and NSF grant No. 1609916.
.}}

\maketitle

\begin{abstract}
Detection rules have traditionally been designed for rational agents that minimize the Bayes risk (average decision cost). With the advent of crowd-sensing systems, there is a need to redesign binary hypothesis testing rules for behavioral agents, whose cognitive behavior is not captured by traditional utility functions such as Bayes risk. In this paper, we adopt prospect theory based models for decision makers. We consider special agent models namely optimists and pessimists in this paper, and derive optimal detection rules under different scenarios. Using an illustrative example, we also show how the decision rule of a human agent deviates from the Bayesian decision rule under various behavioral models, considered in this paper.
\end{abstract}


\section{Introduction \label{sec: Introduction}}
Cognitive behavior has traditionally been modeled using rationality models, where the human agents are assumed to behave in an unbiased manner. Unbiased decision-makers are often assumed to minimize \emph{Bayes risk}, which is defined as the expected cost of making decisions \cite{Book-Kay-Detection}. However, in the real world, human agents may have a cognitive bias, due to the limited availability of information and/or other complex behaviors such as emotions, loss-aversion and endowment effect \cite{Simon1979,Kahneman1979,Einhorn1981,Barberis2003,Johnson2010}. Such complex agents were successfully modeled by Kahneman and Tversky using \emph{prospect theory} in \cite{Kahneman1979}, where human behavior is modeled using weight and value functions over probabilities and costs respectively. In this paper, we derive optimal decision rules for binary hypothesis testing employed by two special prospect-theory based human agents, namely \emph{optimists} and \emph{pessimists}.

In the past, several efforts have been geared towards validating theoretical models to comprehend decision rules employed by human agents, using experimental data. Recently, researchers have been showing significant interest in the design of complex networked systems where human agents interact with machines effectively so that the system operates with maximal efficiency \cite{Coppin2003}. Particularly, in the context of binary hypothesis testing, there are several \emph{crowdsensing} based applications such as Tomnod and Zooniverse where human volunteers participate in the decision-making process of a proposed task. In some cases, systems are designed to emulate human behavior in order to reduce human effort and intervention. One example is the design of self-driving cars by Google and Uber, which move in traffic alongside human-driven vehicles. In contrast, there are other applications where there is a need to steer/nudge human decisions in order to improve the overall performance of the system \cite{Book-Thaler}. In this paper, we study optimal behavioral rules in human agents within the context of binary hypothesis testing, which is essential to propose a design framework, where human decisions can be either emulated/steered in a controlled manner in any human-machine interaction system. 

Traditionally, binary hypothesis testing has been extensively studied by researchers over several decades under different scenarios \cite{Book-Kay-Detection, Book-Varshney}. In particular, Bayesian detection rules are designed under the premise that the decision-maker is rationally motivated to minimize its Bayes risk. With the advent of novel systems which include human decision makers, there is a need to redesign detection systems with human decision-makers. Design of such systems has gained interest, some of such studies are discussed below. In \cite{Rhim2012}, Rhim \emph{et al.} have modeled the problem of distributed hypothesis testing as a \emph{categorical decision-making} problem. Inspired from the observation that human agents make decisions categorically, the authors have modeled the concept of categorization via quantization of prior probabilities. In \cite{Wimalajeewa2013}, Wimalajeewa and Varshney have investigated the problem of collaborative human decision making, where human agents are modeled as likelihood-ratio decision rules with random thresholds. More recently, Vempaty \emph{et al.} have proposed a coding theory based framework in \cite{Vempaty2014} to improve the performance of a distributed detection network with human agents. Later, in \cite{Vempaty2015}, Vempaty \emph{et al.} have proposed a Bayesian hierarchical model for the distributed detection framework to capture the uncertainty in the behavior of human teams. 

In contrast to the past work, we consider the problem of decision making by an individual human agent in the context of binary hypothesis testing. We assume that the decision-making of these individual human agents can be represented by prospect theory models, and therefore, they make decisions that minimize their behavioral risk, which is defined using prospect theory. To the best of our knowledge, this is the first work that considers the problem of prospect theory based decision making for binary hypothesis testing. We investigate the structure of optimal decision rules employed by two special types of behavioral agents, namely optimists and pessimists \cite{Diecidue2001}, under two different conditions depending on the decision costs incurred at the agent. 

\section{Problem Setup \label{sec: Prob-Setup}} 

\begin{figure}[!t]
	\centering
    \includegraphics[width=3.5in]{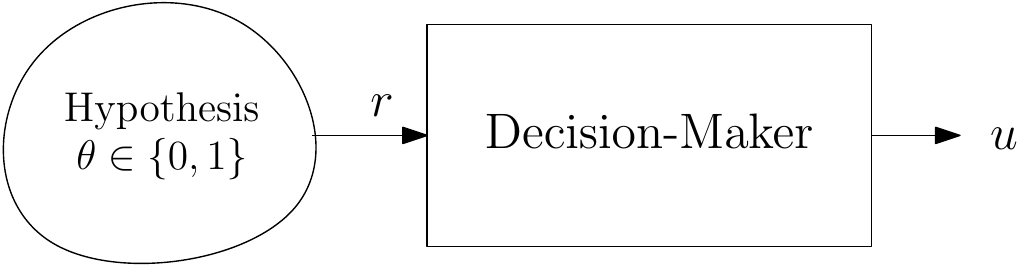}
    \caption{Framework for Binary Hypothesis-Testing}
    \label{Fig: model}
\end{figure}

Consider a binary hypothesis-testing framework with a behavioral decision-maker, as shown in Figure \ref{Fig: model}, where the true hypothesis is denoted by $\theta \in \{0,1\}$. Let $\pi_0$ and $\pi_1$ denote the prior probabilities of the hypotheses $\theta = 0$ and $\theta = 1$ respectively. We assume that the decision-maker receives a real-valued signal $r \in \mathbb{R}$ from the phenomenon-of-interest (PoI) with conditional distribution $p(r|\theta)$ and processes it to make an inference $u$ about the true hypothesis $\theta$. 

\begin{figure*}[!t]
	\centerline
    {
    	\subfloat[Typical weight-function]{\includegraphics[width=3.3in]{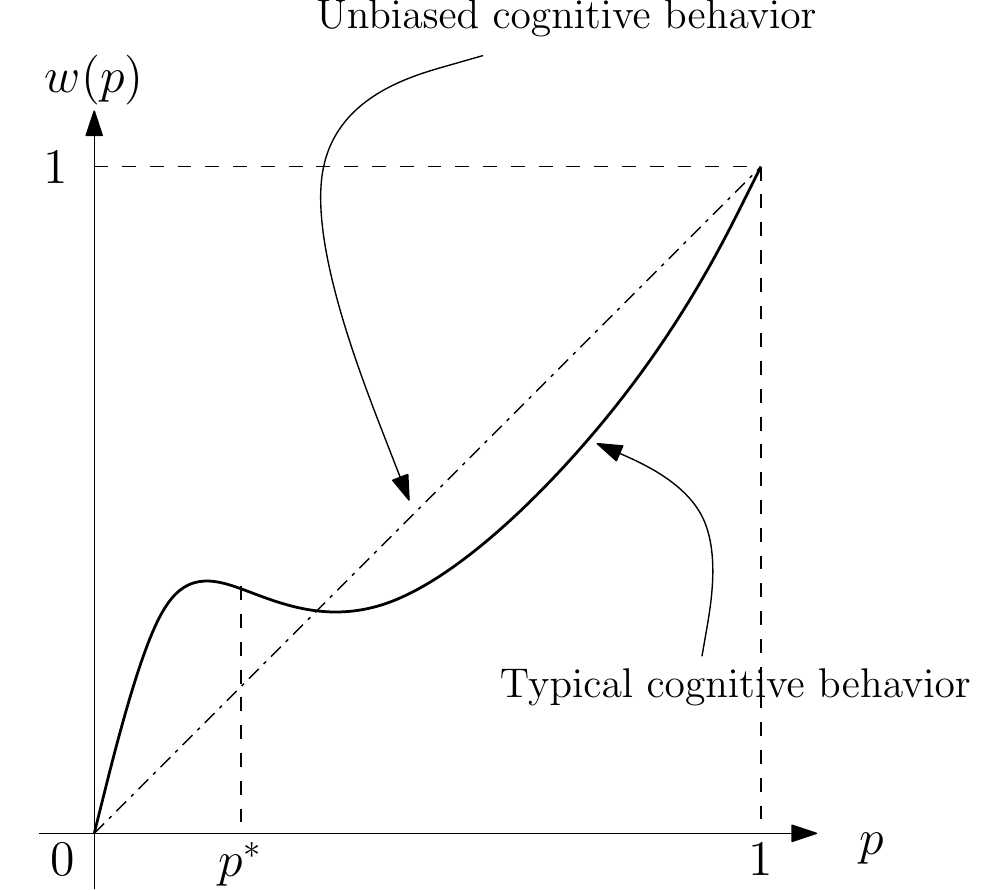}%
        \label{Fig: weight}}
        \hfil
        \subfloat[Typical value-function]{\includegraphics[width=3.3in]{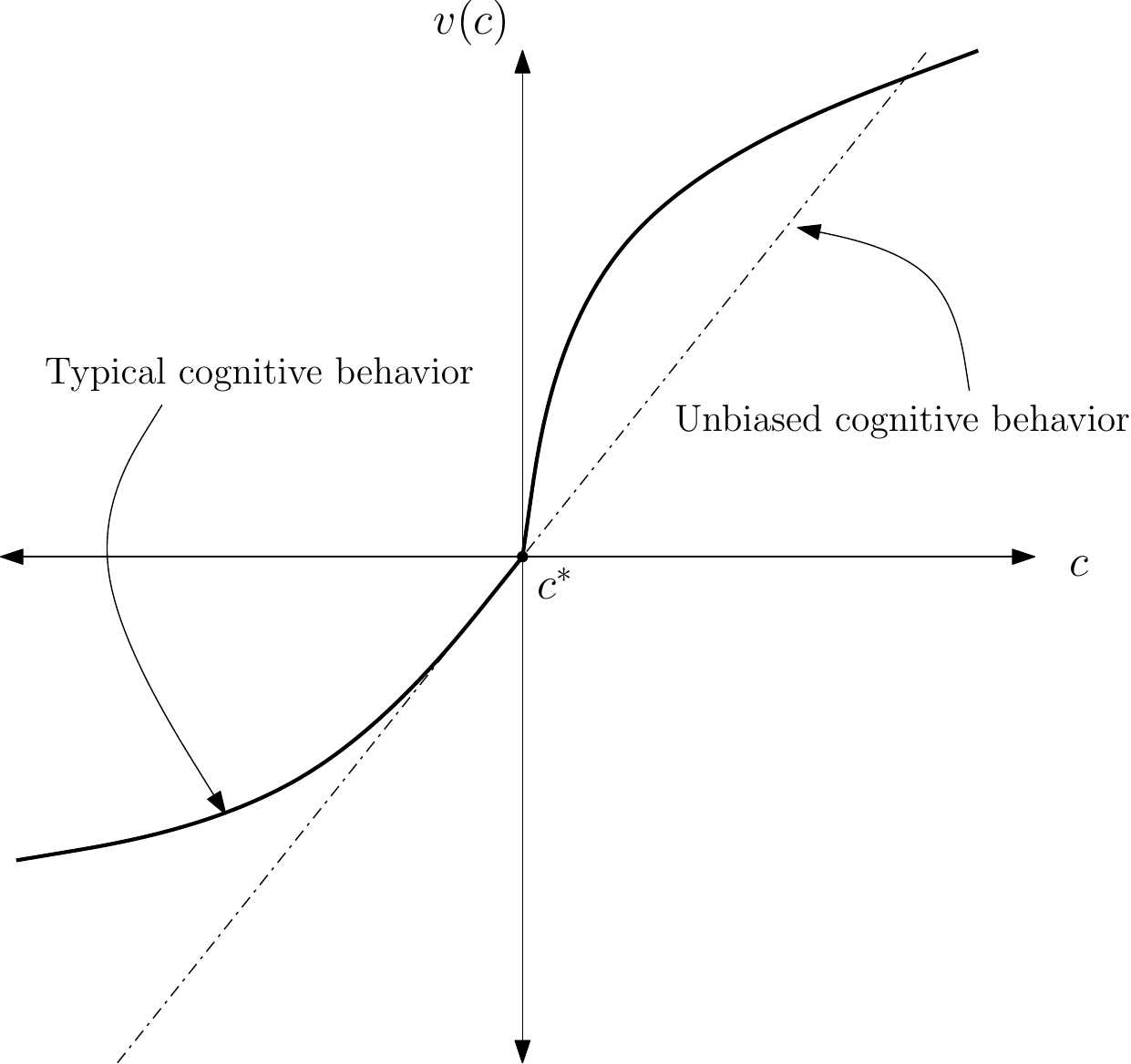}%
        \label{Fig: value}}
    }
    \caption{Modeling Cognitive Bias using Prospect Theory}
    \label{Fig: Prospect-Theory-Model}
\end{figure*}

In this paper, the rationality of the decision-maker is modeled using \emph{prospect theory}\cite{Kahneman1979}, where the behavioral agent cognitively distorts the probabilities and costs using known weight and value functions respectively. In other words, any given probability $\rho$ and any given cost $c$ are perceived as $w(\rho)$ and $v(c)$ respectively, where $w(\cdot)$ is the weight function and $v(\cdot)$ is the value function in the behavioral model. Our goal is to design an optimal decision rule that optimizes the behavioral risk at the decision maker.

In this paper, we assume that the decision maker employs the following decision rule:
\begin{equation}
	u = 
	\begin{cases}
		1; \quad \mbox{if } r \in \mathcal{R}
		\\
		0; \quad \mbox{otherwise},
	\end{cases}
	\label{Eqn: Decision Rule}
\end{equation}
where $\mathcal{R}$ is the acceptance region of the hypothesis $\theta$. The performance of the decision rule, as given in Equation \eqref{Eqn: Decision Rule}, is given by
\begin{subequations}
\begin{equation}
	x = Pr(u = 1 | \theta = 0) = \displaystyle \int_{\mathcal{R}} p(r | \theta = 0) dr,
\end{equation}
\begin{equation}
	y = Pr(u = 1 | \theta = 1)  = \displaystyle \int_{\mathcal{R}} p(r | \theta = 1) dr,
\end{equation}
\label{Eqn: Performance}
\end{subequations}
where $p(r | \theta = 0)$ and $p(r | \theta = 1)$ are the conditional pdfs of the observation $r$ under the two hypotheses $\theta = 0, 1$ respectively. 

Given that the decision maker makes an inference $u = i$ when the true hypothesis is $\theta = j$, we assume that the decision maker incurs a cost $c_{ij}$ for any $i,j \in \{ 0,1 \}$. Therefore, the behavioral risk due to the decision rule given in Equation \eqref{Eqn: Decision Rule}, is defined below:
\begin{equation}
	f(\mathcal{R}) = \displaystyle \sum_{i = 0}^1 \sum_{j = 0}^1 w[Pr(u = i, \theta = j)] \cdot v(c_{ij}).
	\label{Eqn: Risk-Behavioral-def}
\end{equation}

Assuming that the decision maker always wishes to minimize its behavioral risk $f(\mathcal{R})$, we present the following problem statement.
\begin{prob-statement}
	Find the optimal acceptance region $\mathcal{R}^*$ that minimizes $f(\mathcal{R})$ over all possible subsets of $\mathbb{R}$.
	\label{Prob.Stmt: Optimal-Decision-Rule}
\end{prob-statement}

Instead of representing a decision rule using its corresponding acceptance region $\mathcal{R}$, one can also equivalently parameterize the performance of the decision rule using two variables: $x = Pr(u = 1 | \theta = 0)$ and $y = Pr(u = 1| \theta = 1)$, as defined in Equation \eqref{Eqn: Performance}. In such a case, the behavioral risk as given in Equation \eqref{Eqn: Risk-Behavioral-def}, can be rewritten as follows.
\begin{equation}
	f(x,y) = \displaystyle g(x) + h(y),
	\label{Eqn: Risk-Bayes}
\end{equation}
where
\begin{subequations}
\begin{equation}
	g(x) = \displaystyle w[\pi_0(1 - x)] v(c_{00}) + w[\pi_0 x] v(c_{10}),
	\label{Eqn: g(x)}
\end{equation}
\begin{equation}
	h(y) = \displaystyle w[\pi_1 (1-y)] v(c_{01}) + w[\pi_1 y] v(c_{11}).
	\label{Eqn: h(y)}
\end{equation}
\end{subequations}

Furthermore, note that both $x$ and $y$ increase as the size of the region $\mathcal{R}$ increases, since the area under both the conditional distributions $p(x | \theta = 0)$ and $p(x | \theta = 1)$ changes concurrently with the region $\mathcal{R}$. Furthermore, from Caratheodary theorem \cite{Book-Rockafeller1996}, we know that any ROC curve can be made concave by allowing randomization of decision rules. Therefore, in this paper, we assume that $y$ is a concave-increasing function of $x$, for all $0 \leq x \leq 1$, with $y(x = 0) = 0$ and $y(x = 1) = 1$.

Note that the solution to the above problem depends on the weight and the value functions which may even make Problem \ref{Prob.Stmt: Optimal-Decision-Rule} intractable to solve. Therefore, in the following section, we present some basic assumptions on the weight and value functions which have been experimentally verified by several researchers in the past literature \cite{Barberis2003}.


\section{Behavioral Agents and their Properties \label{sec: Prospect Theory}}
Prospect theory models are defined using weight and value functions. While the risk-seeking/risk-averse nature of an agent is captured by the value function, the optimism/pessimism of an agent is modeled using the weight function. For example, the fear of an accident may make the probability of its occurrence seem larger to a human decision-maker, than the true probability of occurrence of an accident. Furthermore, the risk-averse behavior of human agents drives them to overvalue the cost of an accident. Prospect theory has been experimentally studied by several economists and psychologists, and is currently accepted universally to model human behavior. In the rest of this section, we present practical models for both weight and value functions that have been verified extensively in the literature. 

In prospect theory, the distortion of probabilities due to human behavior is modeled via weight functions. These weight functions are bounded within the unit square, with the line $w(p) = p$ denoting the unbiased cognitive behavior. The region above this line corresponds to the region of optimism, while the region below $w(p) = p$ is known as the region of pessimism. In most experimental studies\footnote{For more details about why the weight function has these properties in practice, the readers may refer to \cite{Prelec1998, Gonzalez1999} and references therein.}, as pointed out in Figure \ref{Fig: weight}, the weight functions have been observed to behave in the following manner.
\begin{property}
	There exists a unique $p^* \in [0,1]$ such that the weight function $w(p)$ is concave for all $p < p^*$, and convex for all $p \geq p^*$.
	\label{Assumption: Weight-Function}
\end{property}
Note that, the agent is \emph{optimistic} when $p^* = 1$, and the agent is \emph{pessimistic} when $p^* = 0$ \cite{Diecidue2001}. Furthermore, if the agent is optimistic, since the weight function is always bounded within a unit square, $w(p)$ is concave-increasing for all $p \in [0,1]$. Similarly, in the case of pessimistic agents, the weight function is convex-increasing for all $p \in [0,1]$.

On the other hand, the value functions are distortions that the human agent perceives when they incur a cost (or a reward). These functions are not necessarily bounded, and are observed to be piecewise convex\footnote{Traditionally, value functions are defined over rewards, and therefore, are assumed to be concave. In this paper, we define value functions over costs (or negative rewards) in order to align with the traditional definition of Bayes risk.}, as shown in Figure \ref{Fig: value}. But, in this paper, we relax this to a more general assumption, as defined below.
\begin{property}
	$v(c)$ is continuous and monotonically increasing for any $c \in \mathbb{R}$. Consequently, there exists a unique point $c^* \in \mathbb{R}$ such that $v(c^*) = 0$. 
	\label{Assumption: Value-Function}
\end{property}

Note that the cost $c^*$ at which $v(c^*) = 0$ is well-known as the \emph{reference point}. The reference point $c^*$ plays a major role in the decision-making process, and can be interpreted as a target/goal to any human decision maker \cite{Heath1999}. Consequently, any cost below the reference point $c^*$ appears as profits, and any cost above $c^*$ appears as losses to the human decision maker. Although several human behaviors such as loss aversion are attributed to the convexity of the value function, we have considered some special agent models in this paper which only rely on the continuity and monotonicity of the value function, and the reference point $c^*$.


\section{Optimal Detection Rules for Optimists \label{sec: Optimist-Decision-Rules}}
In this section, we assume that the behavioral agent is optimistic as shown in Figure \ref{Fig: weight-optimist}, and, therefore, the weight function is concave increasing. In such a case, we investigate two different scenarios depending on the decision costs incurred at the agent, as shown below.

\begin{figure}[!t]
	\centering
    \includegraphics[width=3.3in]{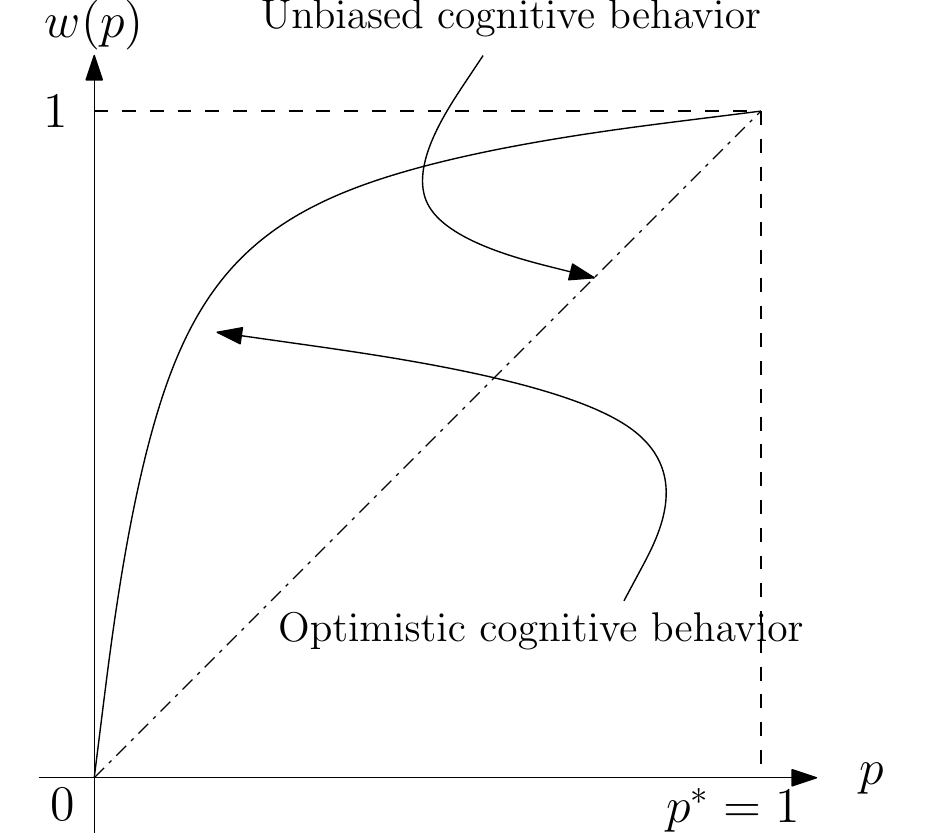}
    \caption{Weight Function for Optimistic Agents}
    \label{Fig: weight-optimist}
\end{figure}

\subsection*{\textbf{\underline{Type-1}}: $ c^* \leq \min \{ c_{00}, c_{01}, c_{10}, c_{11} \} $ \label{sec: Optimist-Type-1}}

This is the case where even an optimist perceives the costs of all the possible choices to be detrimental, i.e., the decision costs of all the available choices are perceived as losses. In other words, given a reference point $c^*$, the optimist perceives a decision cost $c_{ij}$ as $v(c_{ij}) \geq 0$, for all $i, j \in \{ 0,1\}$. As a result, we have the following lemma.

\begin{lemma}
	For all Type-1 optimists, $f(x,y(x))$ is a concave function of $x$, for all $0 \leq x \leq 1$.
	\label{Lemma: f-optimistic-case1}
\end{lemma}
\begin{proof}
We differentiate Equation \eqref{Eqn: g(x)} twice to obtain the following.
\begin{equation}
	\displaystyle \frac{d^2 g(x)}{dx^2} =  \displaystyle \pi_0^2 \left\{ \frac{d^2 w[\pi_0 x]}{dx^2} v(c_{00}) + \frac{d^2 w[\pi_0(1 - x)]}{dx^2} v(c_{10}) \right\}.
\end{equation}
Given that the behavioral agent is optimistic, we have
\begin{equation}
	\begin{array}{lcl}
		\displaystyle \frac{d^2 w[\pi_0 x]}{dx^2} \leq 0, & \mbox{and} & \displaystyle \frac{d^2 w[\pi_0(1 - x)]}{dx^2} \leq 0.
	\end{array}
\end{equation}
Since $v(c_{00})$ and $v(c_{10})$ are both non-negative, we have $\displaystyle \frac{d^2 g(x)}{dx^2} \leq 0$. 

Since both $g(x)$ and $h(y)$ have similar structure, we also have $\displaystyle \frac{d^2 h(y)}{dy^2} \leq 0$. Given that $y$ is a concave increasing function of $x$, $h(y(x))$ is a concave function of $x$. Since $f(x,y(x))$ is the sum of two concave functions $g(x)$ and $h(y(x))$, the behavioral risk is a concave function of $x$.
\end{proof}

Given that the behavioral agent wishes to minimize its risk $f(x,y(x))$, the optimal decision rule lies on the extreme point of the ROC curve, i.e., $(x,y) = (0,0)$ or $(x,y) = (1,1)$. Furthermore, if $c_{00} = c_{11} = c_L \leq c_U = c_{01} = c_{10}$, we have
\begin{subequations}
\begin{equation}
	f(0,0) = w(\pi_0) v(c_L) + w(\pi_1) v(c_U),
\end{equation}
\begin{equation}
	f(1,1) = w(\pi_0) v(c_U) + w(\pi_1) v(c_L).
\end{equation}
\end{subequations}
As a result, if $\pi_0 \geq \displaystyle \frac{1}{2}$, we have $w(\pi_0) \geq w(\pi_1)$. So, we have
\begin{equation}
	f(1,1) - f(0,0) = \left[ w(\pi_0) - w(\pi_1) \right] \cdot \left[ v(c_U) - v(c_L) \right] \geq 0.
\end{equation}
In other words, the behavioral agent chooses the operating point $(1,1)$ in the ROC, which is equivalent to a decision rule where $\mathcal{R} = \mathbb{R}$, i.e., the behavioral agent always decides $u = 1$.

On the other hand, if $\pi_0 < \displaystyle \frac{1}{2}$, we have 
\begin{equation}
	f(1,1) - f(0,0) = \left[ w(\pi_0) - w(\pi_1) \right] \cdot \left[ v(c_U) - v(c_L) \right] \leq 0.
\end{equation}
As a result, the behavioral agent adopts the operating point $(0,0)$. This is equivalent to the case where the behavioral agent always decides $u = 0$. 

In summary, we have the following theorem.
\begin{thrm}
	If $c^* \leq \min \{ c_{00}, c_{01}, c_{10}, c_{11} \}$, an optimist minimizes its behavioral risk by either always deciding $u = 0$, or $u = 1$, for any observation $x \in \mathbb{R}$.
	
	Furthermore, if $c_{00} = c_{11} = c_L \leq c_U = c_{01} = c_{10}$, then a Type-1 optimist employs the following decision rule.
	\begin{equation}
		u = 
		\begin{cases}
			1, \mbox{ if } \pi_0 \geq \displaystyle \frac{1}{2}
			\\
			0, \mbox{ otherwise.}
		\end{cases}
	\end{equation}
	\label{Thrm: Type-1-optimists}
\end{thrm}

In summary, when all the decision costs appear detrimental to an optimist, the optimal decision rule is fixed, and independent of data. Furthermore, in concurrence to our intuition, whenever $c_{00} = c_{11} = c_L \leq c_U = c_{01} = c_{10}$, the optimist optimally chooses the option $u_0 \in \{ 0, 1 \}$ that is antipodal to prior probabilities in order to minimize its behavioral risk.

\subsection*{\textbf{\underline{Type-2}}: $ c^* \geq \max \{ c_{00}, c_{01}, c_{10}, c_{11} \} $ \label{sec: Optimist-Type-3}}
 
In contrast to Type-1 optimists, Type-2 optimists interpret the same decision costs as being profitable. In other words, given that the reference point $c^*$ lies above all the decision costs, $v(c_{ij}) \leq 0$, for all $i, j \in \{ 0,1\}$. As a result, we have the following lemma.

\begin{lemma}
	For all Type-2 optimists, $f(x,y(x))$ is a convex function of $x$, for all $0 \leq x \leq 1$.
	\label{Lemma: f-optimistic-case3}
\end{lemma}
\begin{proof}
	Proof is similar to that of Lemma \ref{Lemma: f-optimistic-case1}, and is therefore, omitted for the sake of brevity.
\end{proof}

In other words, the behavioral agent minimizes its risk at some intermediate operating point $(x^*, y^*)$. In the following theorem, we state the necessary condition that $(x^*, y^*)$ will satisfy. We find this condition by equating the first derivative of the risk function $f(x,y(x))$ with respect to x to zero.

\begin{thrm}
The operating point $(x^*,y^*)$ of the optimal decision rule employed by a Type-2 optimist is a root of the equation
\begin{equation}
	\begin{array}{l}
		\displaystyle \frac{d w[\pi_0 x^*]}{dx} v(c_{10}) + \frac{d w[\pi_1 y^*(x^*)]}{dx} v(c_{11}) 
		\\[2ex]
		\quad = \displaystyle \frac{d w[\pi_0(1 - x^*)]}{dx} v(c_{00}) + \frac{d w[\pi_1\{1 - y^*(x^*)\}]}{dx} v(c_{01}).
	\end{array}
\end{equation}
\label{Thrm: Type-3-optimists}
\end{thrm}

Given that all the decision costs appear profitable to a Type-2 optimist, it is intuitive that the agent chooses a decision rule that minimizes its behavioral risk. Theorem \ref{Thrm: Type-3-optimists} presents the necessary condition to find the optimal operating point for the Type-2 optimist.


\section{Optimal Detection Rules for Pessimists \label{sec: Pessimist-Decision-Rules}}
In this section, we assume that the behavioral agent is pessimistic as shown in Figure \ref{Fig: weight-pessimist}, and therefore, the weight function is convex increasing. More specifically, we analyze the same two types of agents as discussed in Section \ref{sec: Optimist-Decision-Rules}, within the context of pessimists.  

\begin{figure}[!t]
	\centering
    \includegraphics[width=3.3in]{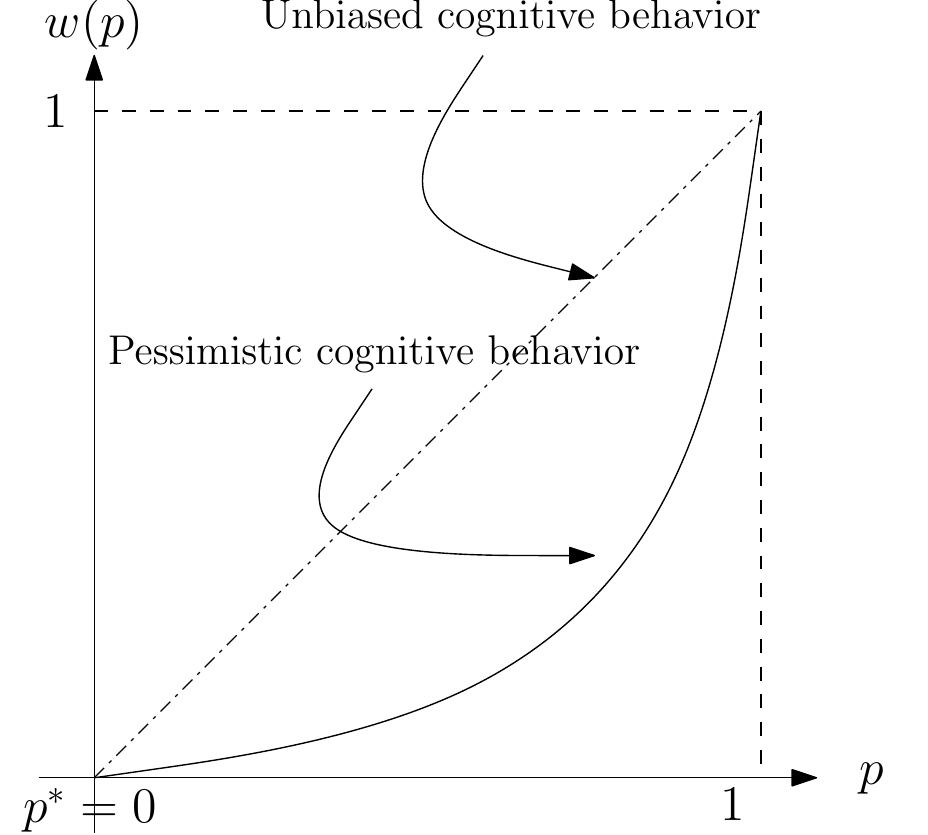}
    \caption{Weight Function for Pessimistic Agents}
    \label{Fig: weight-pessimist}
\end{figure}

\subsection*{\textbf{\underline{Type-1}}: $ c^* \leq \min \{ c_{00}, c_{01}, c_{10}, c_{11} \} $ \label{sec: Pessimist-Type-1}}

As discussed in Section \ref{sec: Optimist-Type-1}, we assume that a Type-1 pessimist also finds all its possible decision choices to be detrimental. In other words, given that $c^*$ lies below all the decision costs, $v(c_{ij}) \geq 0$, for all $i, j \in \{ 0,1\}$. As a result, we have the following lemma.

\begin{lemma}
	For all Type-1 pessimists, $f(x,y(x))$ is a convex function of $x$, for all $0 \leq x \leq 1$.
	\label{Lemma: f-pessimistic-case1}
\end{lemma}
\begin{proof}
We differentiate Equation \eqref{Eqn: g(x)} twice to obtain the following.
\begin{equation}
	\displaystyle \frac{d^2 g(x)}{dx^2} =  \displaystyle \pi_0^2 \left\{ \frac{d^2 w[\pi_0 x]}{dx^2} v(c_{00}) + \frac{d^2 w[\pi_0(1 - x)]}{dx^2} v(c_{10}) \right\}.
\end{equation}
Given that the behavioral agent is pessimistic, we have
\begin{equation}
	\begin{array}{lcl}
		\displaystyle \frac{d^2 w[\pi_0 x]}{dx^2} \geq 0, & \mbox{and} & \displaystyle \frac{d^2 w[\pi_0(1 - x)]}{dx^2} \geq 0.
	\end{array}
\end{equation}
Since $v(c_{00})$ and $v(c_{10})$ are both non-negative, we have $\displaystyle \frac{d^2 g(x)}{dx^2} \geq 0$. 

Since both $g(x)$ and $h(y)$ have a similar structure, we also have $\displaystyle \frac{d^2 h(y)}{dy^2} \geq 0$. Given that $y$ is a convex increasing function of $x$, $h(y(x))$ is a convex function of $x$. Since $f(x,y(x))$ is the sum of two convex functions $g(x)$ and $h(y(x))$, the behavioral risk is a convex function of $x$.
\end{proof}

In other words, the behavioral agent minimizes its risk at some intermediate operating point $(x^*, y^*)$, which satisfies the following necessary condition. We find this condition by equating the first derivative of the risk function $f(x,y(x))$ with respect to $x$ to zero.

\begin{thrm}
The operating point $(x^*,y^*)$ of the optimal decision rule employed by a Type-1 pessimist is a root of the equation
\begin{equation}
	\begin{array}{l}
		\displaystyle \frac{d w[\pi_0 x^*]}{dx} v(c_{10}) + \frac{d w[\pi_1 y^*(x^*)]}{dx} v(c_{11}) 
		\\[2ex]
		\quad = \displaystyle \frac{d w[\pi_0(1 - x^*)]}{dx} v(c_{00}) + \frac{d w[\pi_1\{1 - y^*(x^*)\}]}{dx} v(c_{01}).
	\end{array}
\end{equation}
\label{Thrm: Type-1-Pessimists}
\end{thrm}

Note that our results in Theorem \ref{Thrm: Type-1-Pessimists} are contrary to that of a Type-1 optimist, and resemble that of Theorem \ref{Thrm: Type-3-optimists}. Although all the decision costs appear detrimental, due to the pessimistic nature of the agent, the agent will attempt to find the optimal rule that satisfies the necessary condition stated in Theorem \ref{Thrm: Type-1-Pessimists}.

\subsection*{\textbf{\underline{Type-2}}: $c^* \geq \max \{ c_{00}, c_{01}, c_{10}, c_{11} \}$ \label{sec: Pessimist-Type-3}}

As stated in Section \ref{sec: Optimist-Type-3}, we assume that all the decision costs appear profitable to the agent. In other words, the agent's reference point $c^*$ lies above all the decision costs, $v(c_{ij}) \leq 0$, for all $i, j \in \{ 0,1\}$. As a result, we have the following lemma.

\begin{lemma}
	For all Type-2 pessimists, $f(x,y(x))$ is a concave function of $x$, for all $0 \leq x \leq 1$.
	\label{Lemma: f-pessimistic-case3}
\end{lemma}
\begin{proof}
	Proof is similar to that of Lemma \ref{Lemma: f-optimistic-case1}, and is therefore, omitted for the sake of brevity.
\end{proof}

Given that the behavioral agent wishes to minimize its risk $f(x,y(x))$, the optimal decision rule lies on the extreme point of the ROC curve, i.e., $(x,y) = (0,0)$ or $(x,y) = (1,1)$. Furthermore, if $c_{00} = c_{11} = c_L \leq c_U = c_{01} = c_{10}$, we have
\begin{subequations}
\begin{equation}
	f(0,0) = w(\pi_0) v(c_L) + w(\pi_1) v(c_U),
\end{equation}
\begin{equation}
	f(1,1) = w(\pi_0) v(c_U) + w(\pi_1) v(c_L).
\end{equation}
\end{subequations}
As a result, if $\pi_0 \geq \displaystyle \frac{1}{2}$, we have $w(\pi_0) \geq w(\pi_1)$. So, we have
\begin{equation}
	f(1,1) - f(0,0) = \left[ w(\pi_0) - w(\pi_1) \right] \cdot \left[ v(c_U) - v(c_L) \right] \geq 0.
\end{equation}
In other words, the behavioral agent chooses the operating point $(1,1)$ in the ROC, which is equivalent to a decision rule where $\mathcal{R} = \mathbb{R}$, i.e., the behavioral agent always decides $u = 1$.

On the other hand, if $\pi_0 < \displaystyle \frac{1}{2}$, we have 
\begin{equation}
	f(1,1) - f(0,0) = \left[ w(\pi_0) - w(\pi_1) \right] \cdot \left[ v(c_U) - v(c_L) \right] \leq 0.
\end{equation}
As a result, the behavioral agent adopts the operating point $(0,0)$. This is equivalent to the case where the behavioral agent always decides $u = 0$. 

In summary, we have the following theorem.
\begin{thrm}
	If $c^* \geq \min \{ c_{00}, c_{01}, c_{10}, c_{11} \}$, a pessimist minimizes its behavioral risk by either always deciding $u = 0$, or $u = 1$, for any observation $x \in \mathbb{R}$.
	
	Furthermore, if $c_{00} = c_{11} = c_L \leq c_U = c_{01} = c_{10}$, then a Type-2 pessimist employs the following decision rule.
	\begin{equation}
		u = 
		\begin{cases}
			1, \mbox{ if } \pi_0 \geq \displaystyle \frac{1}{2}
			\\
			0, \mbox{ otherwise.}
		\end{cases}
	\end{equation}
	\label{Thrm: Type-3-Pessimists}
\end{thrm}

In summary, even though a Type-2 pessimist finds all the decision costs to be profitable, the agent still chooses a data-independent decision rule. Particularly, when $c_{00} = c_{11} = c_L \leq c_U = c_{01} = c_{10}$, the Type-2 pessimist chooses a decision that is antipodal to the prior information, just as in the case of a Type-1 optimist.

\section{Illustrative Example \label{sec: Example}}

\begin{figure*}[!t]
	\centerline
    {
    	\subfloat[$\pi_0 = 0.25$]{\includegraphics[width=0.49\textwidth]{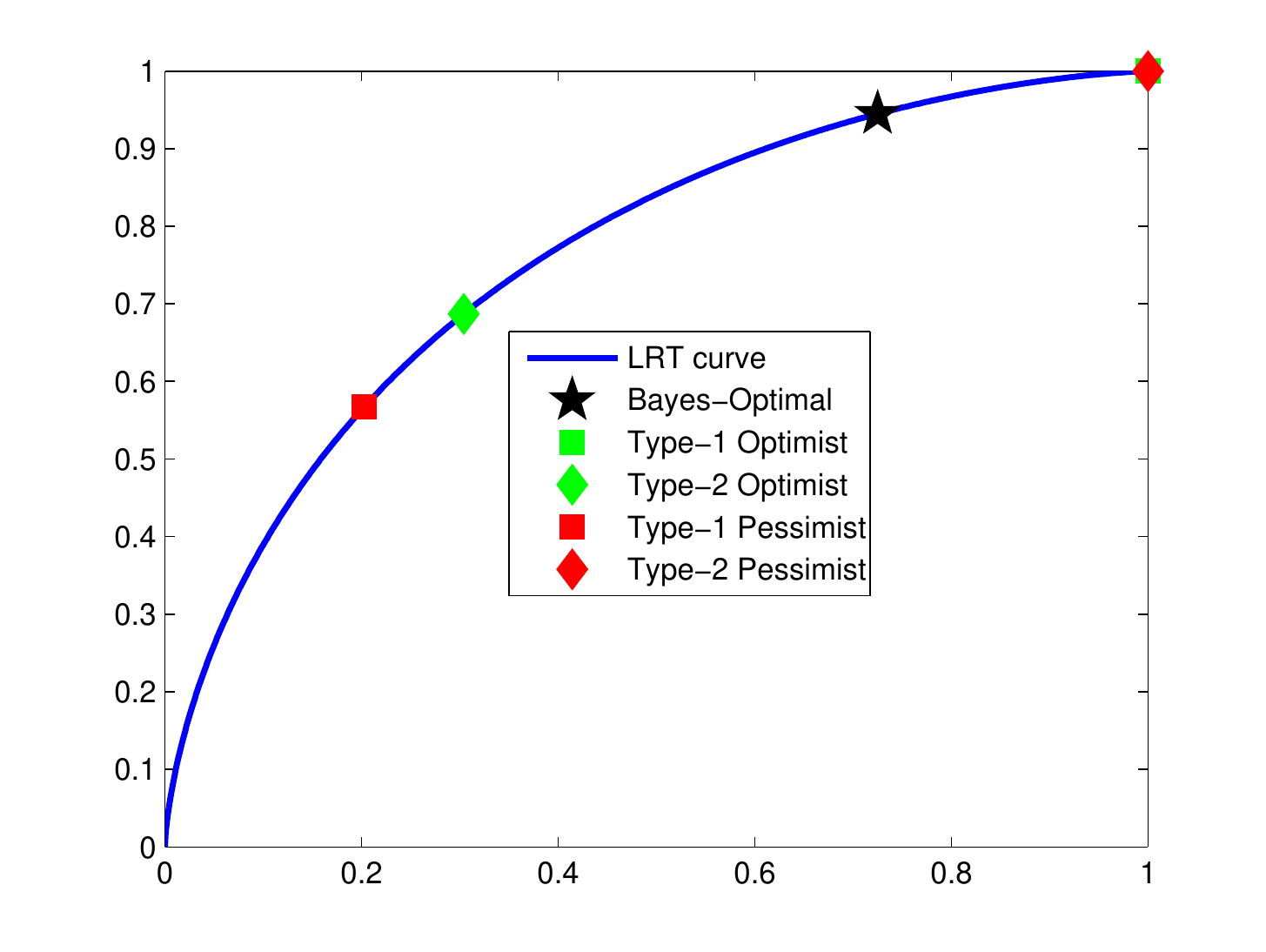}%
        \label{Fig: Low-pi0-result}}
        \hfil
        \subfloat[$\pi_0 = 0.75$]{\includegraphics[width=0.49\textwidth]{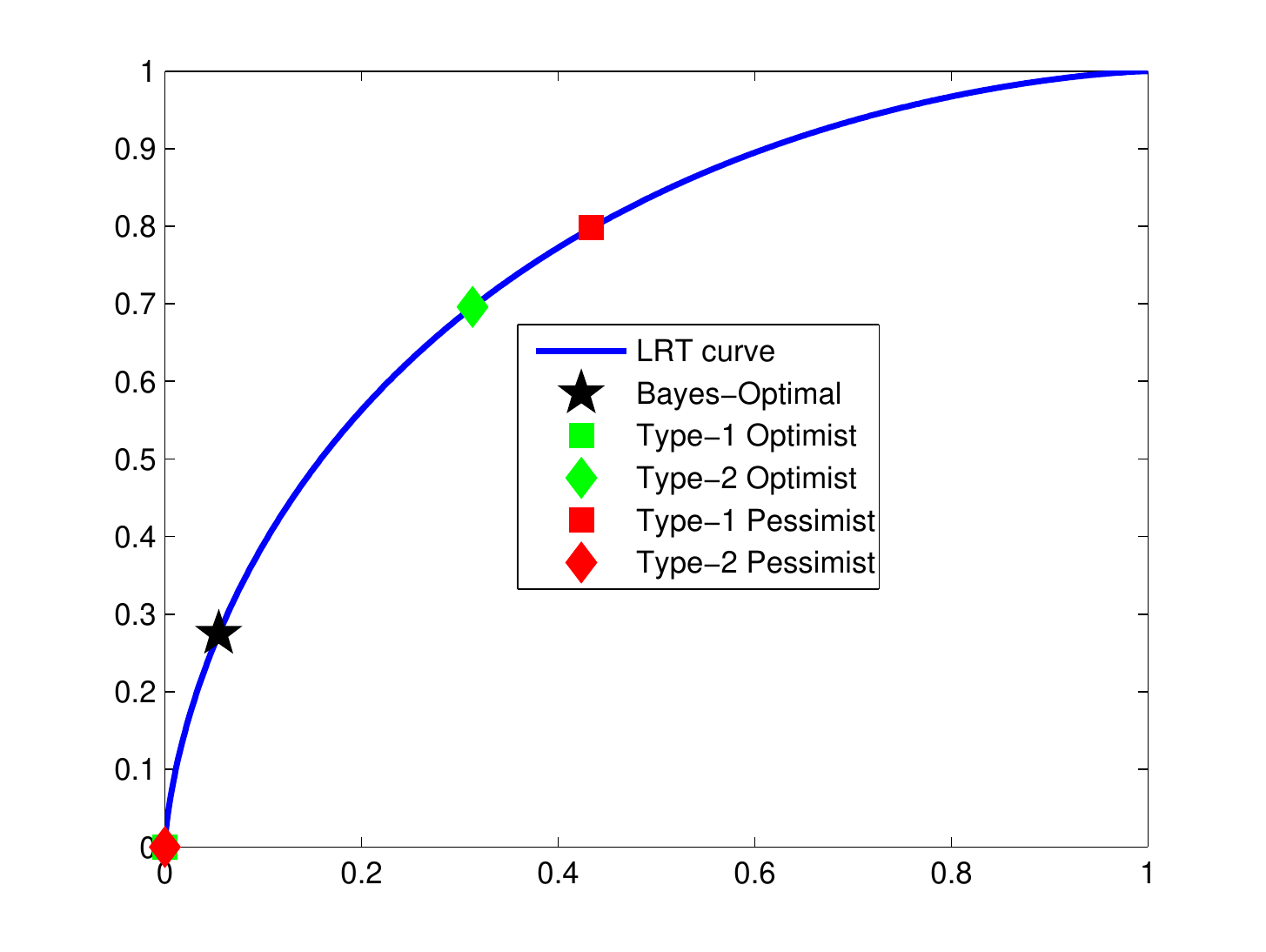}%
        \label{Fig: High-pi0-result}}
    }
    \caption{Optimal decision rules employed by the human agent under different prior probabilities}
    \label{Fig: Results}
\end{figure*}

In this section, we consider a simple example where the observation $r$ follows a normal distribution with mean $\theta$ and unit variance, where $\theta$ denotes the state of the true hypothesis. We assume that the true decision costs incurred by the human agent are given by $c_{00} = c_{11} = c_L = -1$ and $c_{01} = c_{10} = c_U = 1$. In such a case, the optimal Bayesian detection rule is given by
\begin{equation}
	u = 
	\begin{cases}
		1; \quad \mbox{if } x \geq \displaystyle \frac{\pi_0}{\pi_1}
		\\
		0; \quad \mbox{otherwise},
	\end{cases}
	\label{Eqn: Bayesian Decision Rule - example}
\end{equation}
and the corresponding optimal Bayes risk is given by
\begin{equation}
\begin{array}{lcl}
	\mathcal{B}(x,y) & = & \displaystyle \sum_{i = 0}^1 \sum_{j = 0}^1 Pr(u = i, \theta = j) \cdot c_{ij}
	\\[3ex]
	& = & \pi_0 [x \cdot 1 + (1-x) \cdot (-1)] 
	\\
	&& \qquad + \pi_1 [(1-y) \cdot 1 + y \cdot (-1)]
	\\[2ex]
	& = & \pi_0 (2x-1) + \pi_1 (1-2y),
\end{array}
	\label{Eqn: Risk-Bayes-def}
\end{equation}
where $x = Q \left( \displaystyle \frac{\pi_0}{\pi_1} \right)$ and $y = Q \left( \displaystyle \frac{\pi_0}{\pi_1} - \theta \right)$. 

To illustrate our results obtained in Sections \ref{sec: Optimist-Decision-Rules} and \ref{sec: Pessimist-Decision-Rules}, we consider the following model for a human decision-maker and present optimal decision rules for both Type-1 and Type-2 optimists and pessimists. 
\begin{subequations}
\begin{gather}
	w(p) = p^\alpha, \quad \forall \ p \in [0,1] \mbox{ and } \alpha > 0,
	\\[2ex]
	v(c) = e^c - e^{c^*}, \quad \forall \ c \in \mathbb{R},
\end{gather}
\end{subequations}
where $\alpha$ and $c^*$ are the agent's behavioral parameters. More specifically, $\alpha > 1$ if the agent is a pessimist, $\alpha = 1$ if the agent is unbiased, and $0 < \alpha < 1$ if the agent is an optimist. Similarly, $c^*$ is the reference cost of the human decision-maker, as discussed in Section \ref{sec: Prospect Theory}. Assuming that the human agent employs a likelihood-ratio test, we have the following ROC curve.
\begin{equation}
	y = Q \left( Q^{-1}(x) - \theta \right).
\end{equation}

Using this example, we present the operating points adopted by both optimists and pessimists in Figure \ref{Fig: Results}, by obtaining the optimal rules using the gradient descent algorithm. In order to illustrate the results for both Type-1 and Type-2 agents, we present numerical results for two scenarios: (i) $\pi_0 = 0.25$ and (ii) $\pi_0 = 0.75$. Note that, in the first scenario when $\pi_0 = 0.25$, the optimal operating points for both Type-1 optimist and Type-2 pessimist lie at $(1,1)$, which corroborates our theoretical analysis. We also observe a similar behavior with both Type-1 optimists and Type-2 pessimists when $\pi_0 = 0.75$, where the operating points lie at $(0,0)$. In contrast, in the case of Type-2 optimists and Type-1 pessimists, we observe that the optimal decision rules from a behavioral perspective lie at different operating points, and also deviate from the Bayesian detector. Furthermore, since Bayesian decision rules are not necessarily optimal from a behavioral perspective, our results illustrate how prospect theory based decision rules deviate from the Bayesian decision rule.


\section{Conclusion and Future Work \label{sec: Conclusion}}
In this paper, we investigated optimal binary hypothesis testing rules employed by two types of prospect theory based optimists and pessimists. We found that the optimal decision rule employed by an optimist or a pessimist can significantly deviate from the rule which is designed to minimize the Bayes risk. In the future, we will analyze other types of agents and obtain their corresponding optimal decision rules. Note that, in the real world, a typical human agent is neither an optimist, nor a pessimist, and is known to exhibit more complex behavior. Therefore, we will consider behavioral models beyond optimists/pessimists that closely mimic human agents and find optimal decision rules employed by such human agents. 
    
\bibliographystyle{IEEEtran}
\bibliography{IEEEabrv,references}

\begin{thebibliography}{10}
\providecommand{\url}[1]{#1}
\csname url@samestyle\endcsname
\providecommand{\newblock}{\relax}
\providecommand{\bibinfo}[2]{#2}
\providecommand{\BIBentrySTDinterwordspacing}{\spaceskip=0pt\relax}
\providecommand{\BIBentryALTinterwordstretchfactor}{4}
\providecommand{\BIBentryALTinterwordspacing}{\spaceskip=\fontdimen2\font plus
\BIBentryALTinterwordstretchfactor\fontdimen3\font minus
  \fontdimen4\font\relax}
\providecommand{\BIBforeignlanguage}[2]{{%
\expandafter\ifx\csname l@#1\endcsname\relax
\typeout{** WARNING: IEEEtran.bst: No hyphenation pattern has been}%
\typeout{** loaded for the language `#1'. Using the pattern for}%
\typeout{** the default language instead.}%
\else
\language=\csname l@#1\endcsname
\fi
#2}}
\providecommand{\BIBdecl}{\relax}
\BIBdecl

\bibitem{Book-Kay-Detection}
S.~Kay, \emph{Fundamentals of Statistical Signal Processing, Volume II:
  Detection Theory}.\hskip 1em plus 0.5em minus 0.4em\relax Prentice Hall,
  1993.

\bibitem{Simon1979}
H.~A. Simon, ``Rational decision making in business organizations,'' \emph{The
  American economic review}, vol.~69, no.~4, pp. 493--513, 1979.

\bibitem{Kahneman1979}
D.~Kahneman and A.~Tversky, ``Prospect theory: An analysis of decision under
  risk,'' \emph{Econometrica}, vol.~47, no.~2, pp. 263--291, 1979.

\bibitem{Einhorn1981}
H.~J. Einhorn and R.~M. Hogarth, ``Behavioral decision theory: Processes of
  judgment and choice,'' \emph{Journal of Accounting Research}, pp. 1--31,
  1981.

\bibitem{Barberis2003}
N.~Barberis and R.~Thaler, ``A survey of behavioral finance,'' \emph{Handbook
  of the Economics of Finance}, vol.~1, pp. 1053--1128, 2003.

\bibitem{Johnson2010}
J.~G. Johnson and J.~R. Busemeyer, ``Decision making under risk and
  uncertainty,'' \emph{Wiley Interdisciplinary Reviews: Cognitive Science},
  vol.~1, no.~5, pp. 736--749, 2010.

\bibitem{Coppin2003}
G.~Coppin and A.~Skrzyniarz, ``Human-centered processes: individual and
  distributed decision support,'' \emph{IEEE Intelligent Systems}, vol.~18,
  no.~4, pp. 27--33, Jul 2003.

\bibitem{Book-Thaler}
R.~H. Thaler and C.~R. Sunstein, \emph{Nudge: Improving decisions about health,
  wealth and happiness}.\hskip 1em plus 0.5em minus 0.4em\relax Penguin Books,
  February 2009.

\bibitem{Book-Varshney}
P.~K. Varshney, \emph{Distributed Detection and Data Fusion}.\hskip 1em plus
  0.5em minus 0.4em\relax Springer, New York, 1997.

\bibitem{Rhim2012}
J.~B. Rhim, L.~R. Varshney, and V.~K. Goyal, ``Quantization of prior
  probabilities for collaborative distributed hypothesis testing,'' \emph{IEEE
  Transactions on Signal Processing}, vol.~60, no.~9, pp. 4537--4550, Sept
  2012.

\bibitem{Wimalajeewa2013}
T.~Wimalajeewa and P.~K. Varshney, ``Collaborative human decision making with
  random local thresholds,'' \emph{IEEE Transactions on Signal Processing},
  vol.~61, no.~11, pp. 2975--2989, June 2013.

\bibitem{Vempaty2014}
A.~Vempaty, L.~R. Varshney, and P.~K. Varshney, ``Reliable crowdsourcing for
  multi-class labeling using coding theory,'' \emph{IEEE Journal of Selected
  Topics in Signal Processing}, vol.~8, no.~4, pp. 667--679, 2014.

\bibitem{Vempaty2015}
A.~Vempaty, L.~R. Varshney, G.~J. Koop, A.~H. Criss, and P.~K. Varshney,
  ``Decision fusion by people: Experiments, models, and sociotechnical system
  design,'' in \emph{IEEE Global Conference on Signal and Information
  Processing (GlobalSIP)}, Dec 2015, pp. 83--87.

\bibitem{Diecidue2001}
E.~Diecidue and P.~P. Wakker, ``On the intuition of rank-dependent utility,''
  \emph{Journal of Risk and Uncertainty}, vol.~23, no.~3, pp. 281--298, 2001.

\bibitem{Book-Rockafeller1996}
R.~T. Rockafeller, \emph{Convex Analysis}, ser. Princeton Landmarks in
  Mathematics and Physics.\hskip 1em plus 0.5em minus 0.4em\relax Princeton
  University Press, 1996.

\bibitem{Prelec1998}
D.~Prelec, ``The probability weighting function,'' \emph{Econometrica},
  vol.~66, no.~3, pp. 497--527, 1998.

\bibitem{Gonzalez1999}
R.~Gonzalez and G.~Wu, ``On the shape of the probability weighting function,''
  \emph{Cognitive Psychology}, vol.~38, no.~1, pp. 129 -- 166, 1999.

\bibitem{Heath1999}
R.~P.~L. C.~Heath and G.~Wu, ``Goals as reference points,'' \emph{Cognitive
  psychology}, vol.~38, no.~1, pp. 79--109, 1999.

\end{thebibliography}

\end{document}